\documentclass[sigconf]{acmart}

\usepackage{color}
\usepackage{multirow}
\usepackage{algorithm}
\usepackage{algpseudocode}

\usepackage{threeparttable}
\usepackage{caption}
\usepackage{subcaption}
\newcommand{\etal}{\textit{et al.}}
\newcommand{\eg}{\textit{e.g.}}
\usepackage{array}
\usepackage{stfloats}
\usepackage{bbding}

\usepackage{makecell}

\AtBeginDocument{%
  \providecommand\BibTeX{{%
    \normalfont B\kern-0.5em{\scshape i\kern-0.25em b}\kern-0.8em\TeX}}}

\copyrightyear{2022}
\acmYear{2022}
\setcopyright{acmcopyright}\acmConference[MM '22]{Proceedings of the 30th ACM
International Conference on Multimedia}{October 10--14, 2022}{Lisboa, Portugal}
\acmBooktitle{Proceedings of the 30th ACM International Conference on Multimedia
(MM '22), October 10--14, 2022, Lisboa, Portugal}
\acmPrice{15.00}
\acmDOI{10.1145/3503161.3547860}
\acmISBN{978-1-4503-9203-7/22/10}

\begin{document}

\title[Unified Normalization]{Unified Normalization for Accelerating and Stabilizing Transformers}

\author{Qiming Yang}
\orcid{0000-0002-9505-1649}
\email{yangqiming5@hikvision.com}
\affiliation{%
  \institution{Hikvision Research Institute}
  \city{Hangzhou}
  \country{China}
}

\author{Kai Zhang}
\orcid{0000-0002-0581-9571}
\email{zhangkai23@hikvision.com}
\affiliation{%
  \institution{Hikvision Research Institute}
  \city{Hangzhou}
  \country{China}
}

\author{Chaoxiang Lan}
\orcid{0000-0002-5403-5781}
\email{lanchaoxiang@hikvision.com}
\affiliation{%
  \institution{Hikvision Research Institute}
  \city{Hangzhou}
  \country{China}
}

\author{Zhi Yang}
\orcid{0000-0002-3593-1425}
\email{yanzhi13@hikvision.com}
\affiliation{%
  \institution{Hikvision Research Institute}
  \city{Hangzhou}
  \country{China}
}

\author{Zheyang Li}
\orcid{0000-0002-0229-8707}
\email{lizheyang@hikvision.com}
\affiliation{%
  \institution{Hikvision Research Institute \& Zhejiang University}
  \city{Hangzhou}
  \country{China}
}

\author{Wenming Tan}
\orcid{0000-0003-1338-4536}
\email{tanwenming@hikvision.com}
\affiliation{%
  \institution{Hikvision Research Institute}
  \city{Hangzhou}
  \country{China}
}

\author{Jun Xiao}
\orcid{0000-0002-6142-9914}
\email{junx@cs.zju.edu.cn}
\affiliation{%
  \institution{Zhejiang University}
  \city{Hangzhou}
  \country{China}
}

\author{Shiliang Pu}
\authornote{Corresponding author.}
\orcid{0000-0001-5269-7821}
\email{pushiliang.hri@hikvision.com}
\affiliation{%
  \institution{Hikvision Research Institute}
  \city{Hangzhou}
  \country{China}
}

\renewcommand{\shortauthors}{Qiming Yang et al.}

\begin{abstract}
  Solid results from Transformers have made them prevailing architectures in various natural language and vision tasks. As a default component in Transformers, Layer Normalization (LN) normalizes activations within each token to boost the robustness. However, LN requires on-the-fly statistics calculation in inference as well as division and square root operations, leading to inefficiency on hardware. What is more, replacing LN with other hardware-efficient normalization schemes (e.g., Batch Normalization) results in inferior performance, even collapse in training. We find that this dilemma is caused by abnormal behaviors of activation statistics, including large fluctuations over iterations and extreme outliers across layers. To tackle these issues, we propose Unified Normalization (UN), which can speed up the inference by being fused with other linear operations and achieve comparable performance on par with LN. UN strives to boost performance by calibrating the activation and gradient statistics with a tailored fluctuation smoothing strategy. Meanwhile, an adaptive outlier filtration strategy is applied to avoid collapse in training whose effectiveness is theoretically proved and experimentally verified in this paper. We demonstrate that UN can be an efficient drop-in alternative to LN by conducting extensive experiments on language and vision tasks. Besides, we evaluate the efficiency of our method on GPU. Transformers equipped with UN enjoy about \textbf{31\%} inference speedup and nearly \textbf{18\%} memory reduction. Code will be released at \url{https://github.com/hikvision-research/Unified-Normalization}.
\end{abstract}



\maketitle

\section{Introduction}

Transformers~\cite{vaswani2017attention} are initially introduced for Natural Language Processing (NLP) tasks~\cite{devlin2018bert,mehta2020delight}. Since Transformers make few assumptions about the structural bias of input data, these architectures can be universally and flexibly applied in other scenarios, such as multi-modal and speech tasks~\cite{lin2021survey, watanabe2018espnet}. The basic modules of Transformers are stackable multi-head self-attention (MHSA) and feed-forward network (FFN) that enable the capture of long-term dependencies between tokens. In these basic modules, Layer Normalization (LN)~\cite{LayerNorm} is chosen as a default component that releases the training process from heavy dependency on mini-batch samples to handle variable-length input. However, LN requires additional computation and memory overheads during inference because of the on-the-fly statistics, as well as division and square root operations. LN thus is inefficient and hardly meets industrial needs ~\cite{shao2019ssn,yan2020towards}. Nevertheless, replacing LN with Batch Normalization (BN)~\cite{ioffe2015batch} leads to inferior performance in NLP tasks~\cite{shen2020powernorm}. Shen \etal ~\cite{shen2020powernorm} show that the large fluctuations in Transformers within activation statistics result in poor performance. Recently, Transformers are broadly proliferated to Computer Vision (CV) tasks that have led to a series of breakthroughs in image classification, object detection, instance segmentation, etc., also known as Vision Transformers (ViTs)~\cite{carion2020end,dosovitskiy2020image,yuan2021tokens,liu2021swin,dong2021cswin,chen2021visformer,touvron2021going,graham2021levit,yu2021soit}. LN is directly inherited from the original Transformer as an essential component in these ViTs despite the fixed-length inputs. Similar to NLP tasks, significant performance degradation (even collapse in training) will also be triggered by replacing LN with BN in Transformers in CV tasks~\cite{chen2021empirical,shao2021dynamic,yao2021leveraging}. Shao \etal ~\cite{shao2021dynamic} hold the view that BN is harmful to ViTs that result in performance degradation. In previous works~\cite{chen2021empirical, yao2021leveraging}, the authors claim that replacing all LN with BN in ViTs leads to convergence problems. Since the naive replacement for LN leads to inferior performance and instability, the deployment of Transformers still suffers from the on-the-fly statistics computation.

\begin{figure}
  \centering
  \includegraphics[width=0.9\linewidth]{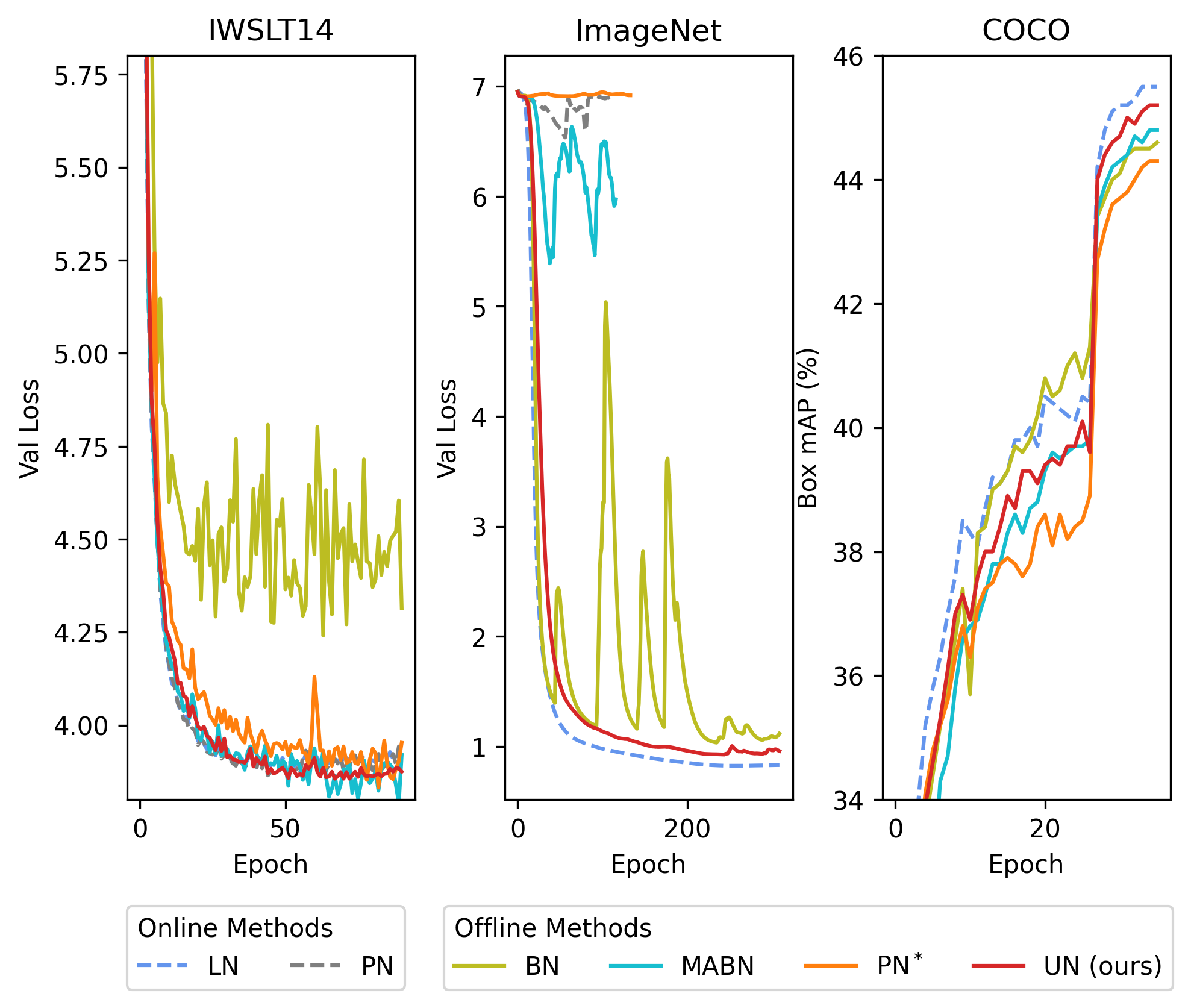}
   \caption{Performance comparison on Transformer (IWSLT14), T2T-ViT-14 (ImageNet), and Swin-T (COCO) during training. The offline methods that could be fused into other linear operations in inference are all plotted with a solid line. PN originally comes with a layer-scale layer, which is removed in PN$^*$.}
  \label{fig:1}
\end{figure}

There are two main methods to improve the hardware efficiency of normalization in Transformers.
1) Simplifying the computation of online statistics~\cite{zhang2019root, shen2020powernorm} and removing inefficient operations (e.g., square root)~\cite{lin2020towards}. Although the performance is maintained, the dynamic calculation for online statistics still exists. 2) Removing the computation for online statistics that utilizes fixed statistics during inference as an offline method, such as BN~\cite{ioffe2015batch}. In this way, inference can be sped up by circumventing the redundant statistics computation getting rid of division and square root operations. As there is no free lunch, significant performance drop and convergence problems are reported in these works~\cite{yao2021leveraging,yan2020towards,shen2020powernorm}. MABN~\cite{yan2020towards} and PN~\cite{shen2020powernorm} utilize moving average strategies to mitigate the fluctuations in Transformers in NLP tasks. However, these mentioned methods are task-specific where the inferior performance and instability still exist in ViTs.

To address the issues above, we dissect the abnormal behaviors of statistics in Transformers. We investigate the activation statistics under moving average strategies in the training of Transformers. We uncover that the fluctuations of the activation statistics are more drastic than that of the gradient statistics during training. Moreover, we find the range of activation statistics task-agnostically keeps increasing along with both the depth and the progress of training, in which the risk of outliers arises. In this sense, extreme outliers are nearly inevitable and continually deteriorate the consistency between activation and gradient statistics. These observations illustrate that the inferior performance and instability (shown in Figure~\ref{fig:1}) are very like boil down to abnormal behaviors of activation statistics in Transformers.

In this paper, we aim to replace LN with an offline method to promote applications for Transformers in language and vision tasks. We propose Unified Normalization (UN) to accelerate inference in Transformers and achieve comparable performance with LN. Specifically, we design a tailored fluctuation smoothing strategy to deal with the fluctuations of different degree in activation and gradient statistics. At the same time, an adaptive outlier filtration strategy is introduced to ensure stable convergence, where the impact of outliers can be proved to be significantly reduced both in theory and experiments. Extensive experiments demonstrate the effectiveness of UN. In a nutshell, our contributions are as follows:
\begin{itemize}
    \item We analyze the abnormal behaviors of activation statistics in Transformers and find the large fluctuations and extreme outliers are responsible for inferior performance and instability.
    
    \item A tailored fluctuation smoothing strategy is designed to calibrate the activation and gradient statistics and boost the performance.
    
    \item An adaptive outlier filtration strategy is introduced to reduce the impact of extreme outliers on the basis of theoretical analysis. 
    
    \item Extensive evaluations in neural machine translation, image classification, object detection, and instance segmentation illustrate the superiority of our method, which is capable of being a drop-in alternative to LN in Transformers. Furthermore, we show that Transformers equipped with UN gain nearly \textbf{18\%} memory reduction and over \textbf{31\%} speedup in inference on GPU.
\end{itemize}

\section{Related Work}

\subsection{Transformers and Vision Transformers}
Transformers~\cite{vaswani2017attention} initially show surprising capability in sequence modeling and neural machine translation. Owing to the high flexibility, Transformers~\cite{devlin2018bert, brown2020language, gulati2020conformer,yao2021wenet} have become the most recent dominant architectures over various NLP tasks and speech tasks. In 2020, Carion \etal~\cite{carion2020end} propose the first end-to-end Transformer-based detector DETR. Later, ViT~\cite{dosovitskiy2020image} is proposed as the very first pure Transformer in CV tasks. The following years have witnessed explosive development of Transformers in CV tasks. T2T~\cite{yuan2021tokens} introduces token-to-token module that combines adjacent tokens in early stage to model local information. Swin~\cite{liu2021swin} and Swin V2~\cite{liu2021swinv2} have applied window-based attention to reduce the overhead of computation in MHSA and achieve state-of-the-art performance. Most recently, Ding \etal~\cite{ding2022davit} explore attention in both spatial and channel tokens to propose a powerful backbone DaViT for vision tasks. In these aforementioned Transformers and Vision Transformers, LN is the preferred choice for normalization. 

In addition, some works~\cite{chen2021visformer,huang2021shuffle,graham2021levit} originally build Transformers with BN, whilst these works need elaborate design on convolutional operations to stabilize training that have considerably modified the original ViT~\cite{dosovitskiy2020image}.

\subsection{Normalization Methods}

\begin{figure}
  \centering
  \includegraphics[width=\linewidth]{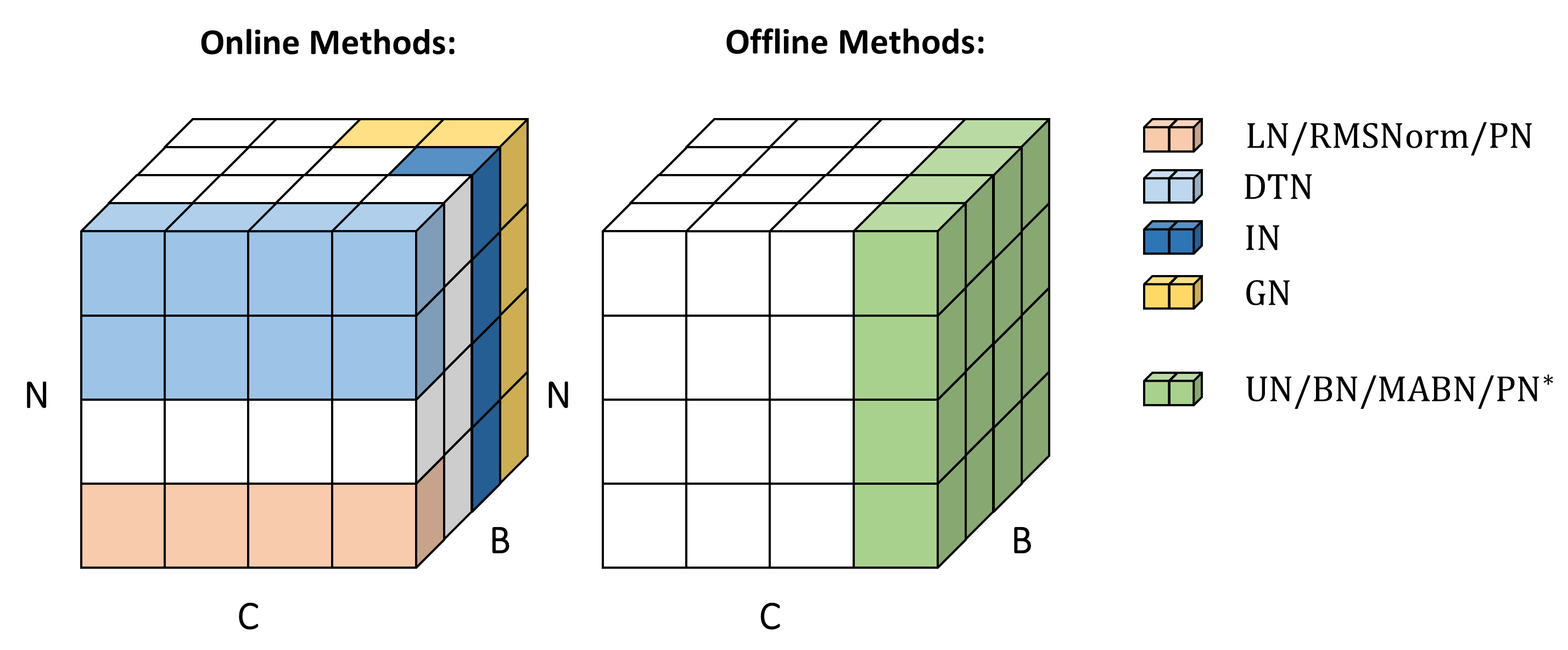}
   \caption{Normalization methods. Each subfigure shows a feature map tensor, where B is the batch axis, N is the number of tokens (or the sequence length) axis, and C is the channel (also known as the embedding size) axis.}
  \label{fig:norm-methods}
\end{figure}

Normalization is widely used for stabilizing training and boosting performance in deep neural networks~\cite{huang2020normalization}. As illustrated in Figure~\ref{fig:norm-methods}, related normalization methods could be categorized into online methods and offline methods according to whether the inference statistics can be fused or not.

\subsubsection{Online Methods}
Online methods require the calculation of on-the-fly statistics during training as well as inference. IN~\cite{ulyanov2016instance}, GN~\cite{wu2018group}, and LN~\cite{LayerNorm} are representative online methods that calculate statistics in different dimensions as shown in Figure~\ref{fig:norm-methods}. Switchable Normalization~\cite{luo2018differentiable} learns to switch between different types of normalization by learning their importance weights. It is widely believed that LN is customized for variable-length NLP samples~\cite{shen2020powernorm}. With the rise of Vision Transformers (ViTs) ~\cite{dosovitskiy2020image,yuan2021tokens,wang2021pyramid,liu2021swin}, LN has also become a preferred choice for CV tasks. Lately, DTN~\cite{shao2021dynamic} exploits the connection within adjacent tokens to improve the performance of LN in ViTs. To make LN more hardware-efficient, some works ~\cite{lin2020towards,xu2019understanding,zhang2019root} attempt to reduce the cost of computation in LN. Zhang \etal~\cite{zhang2019root} propose a simpler method RMSNorm that scales inputs by the root mean square. However, the inefficient dynamic calculation for online statistics is still not fundamentally removed.

\subsubsection{Offline Methods}
Offline methods use estimated inference statistics that could be frozen for arbitrary inputs. Only a point-wise add and multiplication are required during inference that enables the fusion of offline methods with adjacent linear operations. In this way, offline methods can be removed entirely from models and achieve efficient inference~\cite{yan2020towards,yao2021leveraging}. 
However, once these methods cooperate with Transformers, large fluctuations over iterations will lead to performance degradation and even collapse in training~\cite{shen2020powernorm,shao2021dynamic,chen2021empirical,yao2021leveraging}.  Yao \etal~\cite{yao2021leveraging} find Transformers trained with BN are very unstable and crash irregularly. Chen \etal ~\cite{chen2021empirical} attempt to partially replace LN with BN in FFN for stabilizing the training of Transformers. To mitigate the impact of large fluctuations, MABN~\cite{yan2020towards} leverages exponential moving average statistics in activation statistics, and accordingly uses simple moving average statistics in gradient statistics to estimate gradients. Similarly, Shen \etal ~\cite{shen2020powernorm} propose PN$^*$ that uses exponential moving average statistics in both activation and gradient statistics. Preceding works aim to improve the efficiency of LN in Transformers but still suffer from inferior performance and instability. Thus, it is valuable for the community to design a more effective and robust method.

\section{Method}
\label{sec:method}
In this section, we describe the design process of Unified Normalization (UN). First, we develop a unified framework for leveraging offline methods. Based on the framework, we next apply a tailored fluctuation smoothing strategy to mitigate the fluctuations and an adaptive outlier filtration strategy for stabilizing training.

\subsection{Unified Framework}

We develop a unified framework for applying offline methods in Transformers. In this pipeline, the inference statistics are fixed so that they could be fused with other linear operations for speedup.

For a normalization layer, let $\mathbf{X} \in \mathbb{R}^{B\times C}$ and $\mathbf{Y}\in \mathbb{R}^{B\times C}$ denote the input and output, where $B$ is the batch size and $C$ indicates the number of channels. Note that the number of tokens $N$, which could be squeezed into $B$, is omitted in this section for clarity. For arbitrary input in inference, all offline methods perform in a unified manner
\begin{equation}
  \mathbf{Y}=\gamma \cdot \frac{\mathbf{X}-\mu}{\sqrt{\sigma^{2}+\epsilon}} +\beta.
  \label{eq:inference1}
\end{equation}
Here, $\epsilon$ is a small constant, and $\gamma, \beta \in \mathbb{R}^{C}$ are learnable parameters. The inference statistics $\mu, \sigma^{2} \in \mathbb{R}^{C}$ are estimated in the training process and independent of inputs. Since the statistics and parameters are fixed during inference, offline normalization can be merged. The pseudo code for fusing offline normalization with adjacent linear operation can be found in Algorithm~\ref{alg:merge2next}. On the contrary, LN requires calculation for on-the-fly statistics $\mu_{LN}=\mu_{LN}(\mathbf{X}), \sigma_{LN}^{2}=\sigma_{LN}^{2}(\mathbf{X}) $ that consumes extra computation time.

\begin{algorithm}[!t]
\caption{Fusing Normalization}
\label{alg:merge2next}

\begin{algorithmic}[1]
\Require $ \gamma, \beta, \mu, \sigma^2 \in \mathbb{R}^{C}$ 
\hfill \textcolor[rgb]{0.56,0.76,0.9}{$//$ in Equation~\ref{eq:inference1}} \newline
$ \mathbf{W} \in \mathbb{R}^{C_{\text{out}} \times C }, b \in \mathbb{R}^{C_{\text{out}}} $
\hfill \textcolor[rgb]{0.56,0.76,0.9}{$//$ parameters in the subsequent layer}
\Ensure $ \mathbf{W}' \in \mathbb{R}^{ C_{\text{out}} \times C}, b' \in \mathbb{R}^{ C_{\text{out}} } $  \hfill \textcolor[rgb]{0.56,0.76,0.9}{$//$ fused parameters}

\State $ \tilde{\gamma} = \gamma / \sigma $
\State $ \tilde{\beta} = \beta - \tilde{\gamma} \cdot \mu $
\State $ b' = b + \mathbf{W} \times \tilde{\beta} $ \hfill \textcolor[rgb]{0.56,0.76,0.9}{$//$ $y=\mathbf{W}(Norm(x))+b$ is equivalent to}
\State $ \mathbf{W}' = \mathbf{W} \cdot (\textbf{1}_{ C_{\text{out}} } \times  \tilde{\gamma}^{T}) $ \hfill \textcolor[rgb]{0.56,0.76,0.9}{$//$ $ y=\mathbf{W}'x+b' $}

\end{algorithmic}
\end{algorithm}

In forward propagation of training, the normalization procedure is shown as follows,

\begin{equation}
  \mathbf{Z}_{t}=\frac{\mathbf{X}_{t}-\hat{\mu}_t}{\sqrt{\hat{\sigma}_{t}^{2}+\epsilon}},
  \label{eq:fp1}
\end{equation}

\begin{equation}
  \mathbf{Y}_{t}=\gamma \cdot \mathbf{Z}_{t}+\beta.
  \label{eq:fp2}
\end{equation}

\noindent Let $\mathbf{Z}_{t}$ denote the normalized alternative to input $\mathbf{X}_{t}$ at iteration $t$. The training statistics for normalizing are marked as $\hat{\mu}_t$ and $ \hat{\sigma}_{t}^{2}$, given by

\begin{equation}
  \hat{\mu}_t= \Theta_{\mu}(\mu_t, \cdots, \mu_{t-M+1}),
  \label{eq:Theta-mu}
\end{equation}

\begin{equation}
  \hat{\sigma}_{t}^{2}= \Theta_{\sigma^2}(\sigma^2_t, \cdots, \sigma^2_{t-M+1}).
  \label{eq:Theta-sigma}
\end{equation}

\noindent Here, $\mu_t, \cdots, \mu_{t-M+1}$ and $\sigma^2_t, \cdots, \sigma^2_{t-M+1}$ are sequences of recorded statistics from recent $M$ iterations. We consider $\mu_{t}$ and $\sigma_{t}^{2}$ to be the first-moment and second-moment statistics for current input $\mathbf{X}_{t}$. In general, the training statistics can be used to update the inference statistics by applying moving averages.
In backward propagation, the gradients of loss $L$ pass as:

\begin{equation}
  \frac{\partial L}{\partial \mathbf{Z}_{t}} = \gamma\cdot \frac{\partial L}{\partial \mathbf{Y}_{t}},
  \label{eq:bp1}
\end{equation}

\begin{equation}
  \frac{\partial L}{\partial \mathbf{X}_{t}} =\frac{1}{\sqrt{\hat{\sigma}_{t}^{2}+\epsilon} }  (\frac{\partial L}{\partial \mathbf{Z}_{t}}-\psi_{\hat{\mu}_{t}}-\mathbf{Z}_{t}\cdot \psi_{\hat{\sigma}^{2}_{t}}).
  \label{eq:bp2}
\end{equation}

\noindent Giving gradients $\frac{\partial L}{\partial \mathbf{Y}_{t}}$, $\psi_{\hat{\mu}_{t}}$ and $ \psi_{\hat{\sigma}^{2}_{t}}$ indicate the gradient statistics that used for estimating $\frac{\partial L}{\partial \mathbf{X}_{t}}$. In this framework, estimated gradients are gained from averaging functions $\Theta_{g_{\mu}}$ and $\Theta_{g_{\sigma^2}}$,

\begin{equation}
  \psi_{\hat{\mu}_{t}}= \Theta_{g_{\mu}}(g_{\hat{\mu}_t}, \cdots, g_{\hat{\mu}_{t-M+1}}),
  \label{eq:HM-mu}
\end{equation}

\begin{equation}
  \psi_{\hat{\sigma}^{2}_{t}}= \Theta_{g_{\sigma^2}}(g_{\hat{\sigma}^2_t}, \cdots, g_{\hat{\sigma}^2_{t-M+1}}).
  \label{eq:HM-sigma}
\end{equation}

\noindent The gradients passed from $\hat{\mu}_t$ and $\hat{\sigma}_{t}^{2}$ are denoted as $g_{\hat{\mu}_t}$ and $g_{\hat{\sigma}^{2}_{t}}$. 

\noindent \textbf{Offline Methods in the Unified Framework.} With the help of the unified framework, offline methods can be expressed by choosing different statistical objects and averaging functions $\Theta$. 

For instance, BN can be formulated by choosing the mean and variance for the first-moment and second-moment statistics, i.e.,
\begin{equation}
\mu_t=\frac{1}{B}{\textstyle \sum_{i=1}^{B}}\mathbf{x}_{i}, \quad
\sigma_{t}^{2}=\frac{1}{B}{\textstyle \sum_{i=1}^{B}}\mathbf({x}_{i}-\mu_t)^{2}.
\label{eq:BN-stats}
\end{equation}
Then setting averaging functions as:
\begin{equation}
\hat{\mu}_t=\mu_{t}, \quad
\hat{\sigma}^2_t=\sigma^2_t, \quad \psi_{\hat{\mu}_{t}}=g_{\hat{\mu}_t}, \quad \psi_{\hat{\sigma}^{2}_{t}}=g_{\hat{\sigma}^2_t}.
\label{eq:BN-theta}
\end{equation}
Here, the averaging functions simply focus on statistics of current iteration $t$ and ignore the last $M-1$ statistics in the sequences. During training, the inference statistics are updated as,
\begin{equation}
\mu=\alpha \mu + (1-\alpha)\hat{\mu}_t, \quad
\sigma^{2}=\alpha \sigma^{2}+(1-\alpha )\hat{\sigma}_{t}^{2}.
\label{eq:BN-inference}
\end{equation}
As illustrated in Equation~\ref{eq:BN-theta}, BN merely focuses on the activations in the current iteration. This makes BN fragile for large fluctuations over iterations.

For MABN~\cite{yan2020towards}, the authors reduce the number of statistics for stabilizing training and remove the first-moment statistic. Hence, the quadratic mean is chosen as the second-moment statistic:
\begin{equation}
\sigma^2_{t} = \frac{1}{B}{\textstyle \sum_{i=1}^{B}}\mathbf{x}^2_{i}.
\label{eq:MABN-stats}
\end{equation}
The averaging functions for MABN are set as follow,
\begin{equation}
\hat{\mu}_t=0, \quad
\hat{\sigma}^2_t=EMAS\footnote{EMAS (Exponential Moving Average Statistics) for $K$: $K=\eta \cdot K + (1-\eta)\cdot K_t$}, \quad
\psi_{\hat{\mu}_{t}}=0, \quad \psi_{\hat{\sigma}^{2}_{t}}=SMAS\footnote{SMAS (Simple Moving Average Statistics) for $K$ with a window size $M$: $K=\frac{1}{M} {\textstyle \sum_{i=0}^{M-1} K_{t-i}}$}.
\label{eq:MABN-theta}
\end{equation}
The inference statistics $\sigma^{2}$ for MABN are updated the same as Equation~\ref{eq:BN-inference}. With solely applying EMAS in activation statistics, it is hard for MABN to avoid influence from extreme outliers.

\subsection{Fluctuation Smoothing}

\begin{figure}
  \centering
 \includegraphics[width=0.495\linewidth]{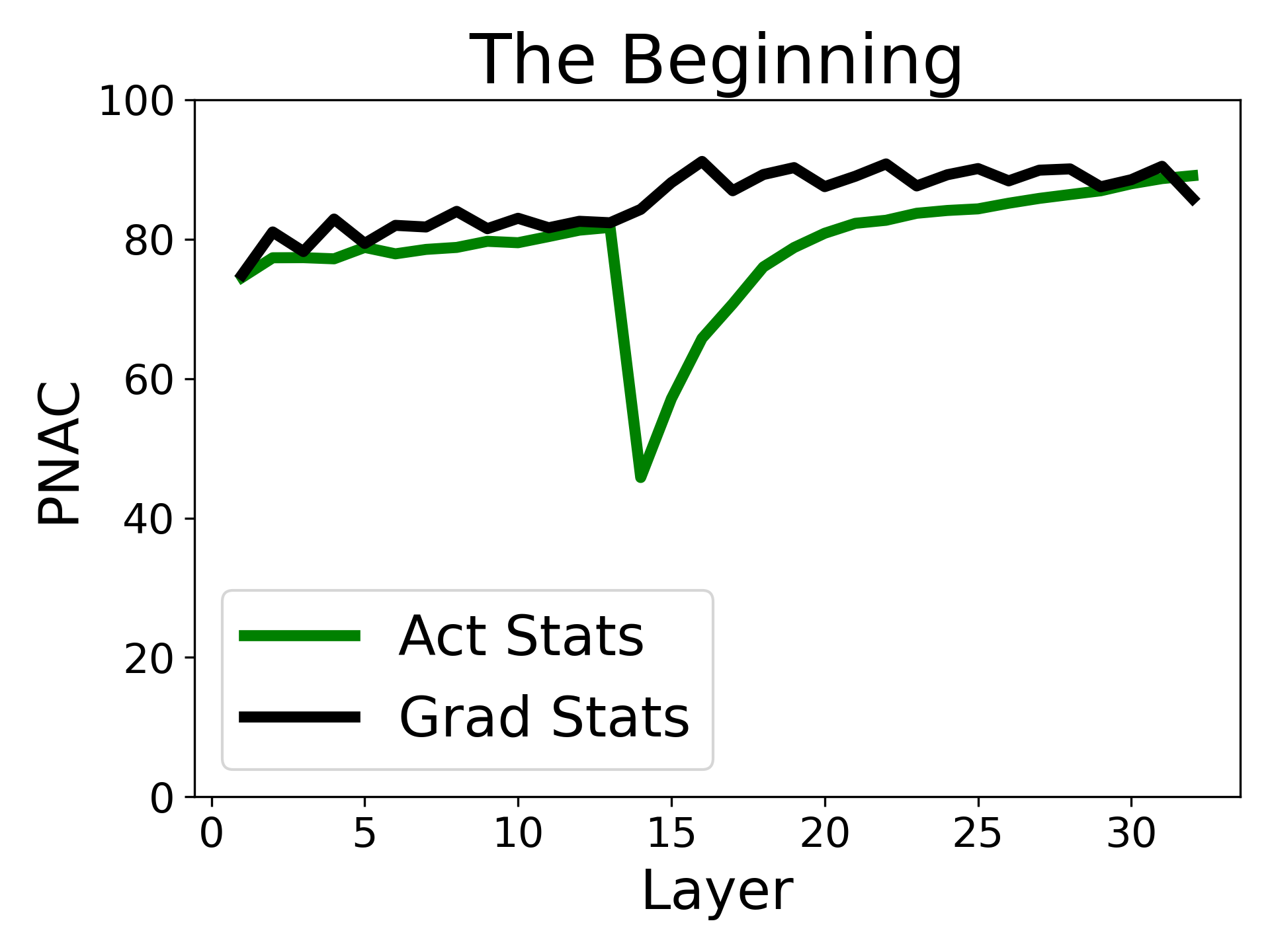}
 \includegraphics[width=0.495\linewidth]{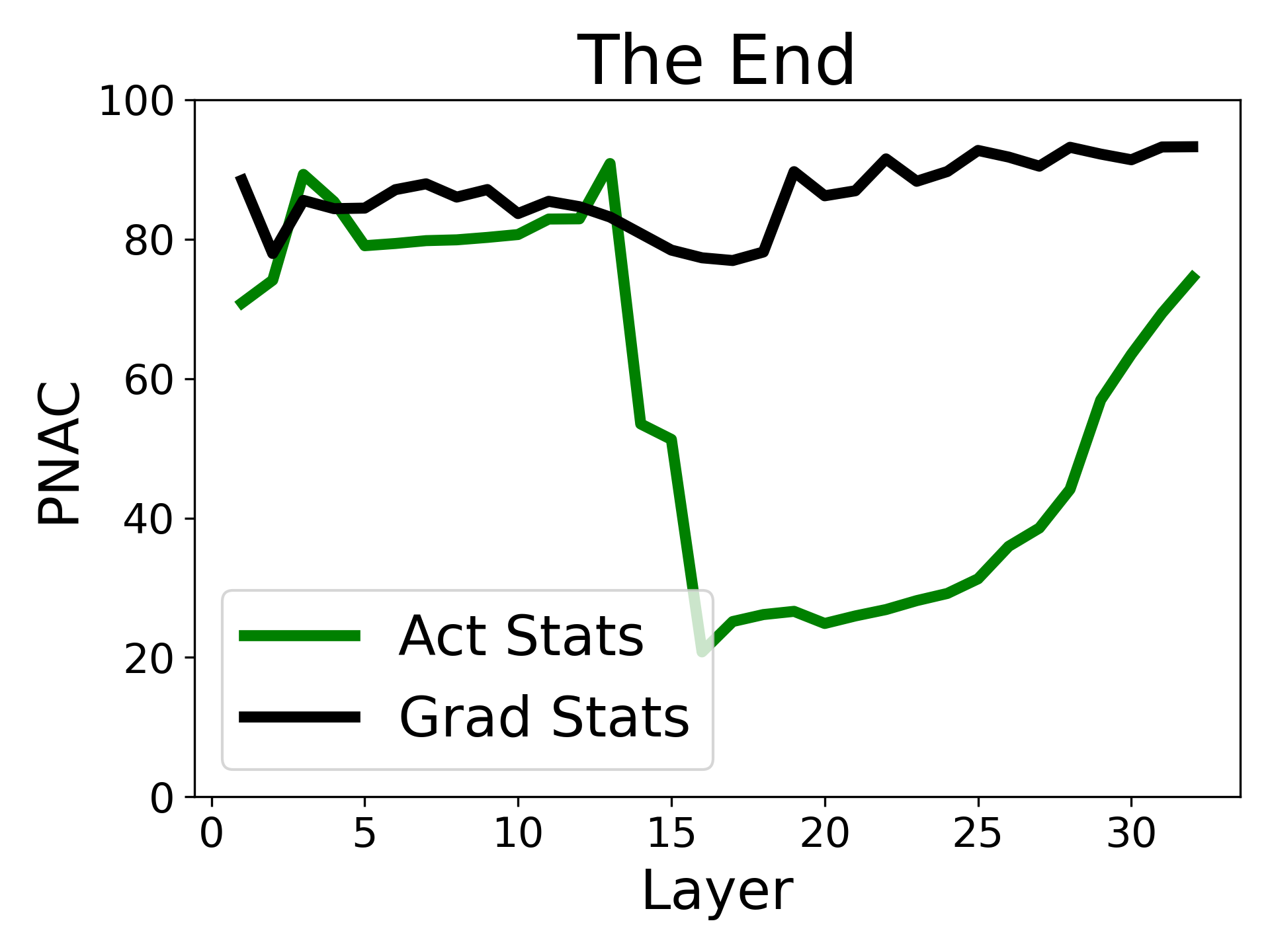}
  \caption{The average PNAC of activation and gradient statistics over iterations in Transformer. A higher PNAC indicates milder fluctuations.}
  \label{fig:2-a}
\end{figure}

\begin{figure*}
  \centering
 \includegraphics[width=\linewidth]{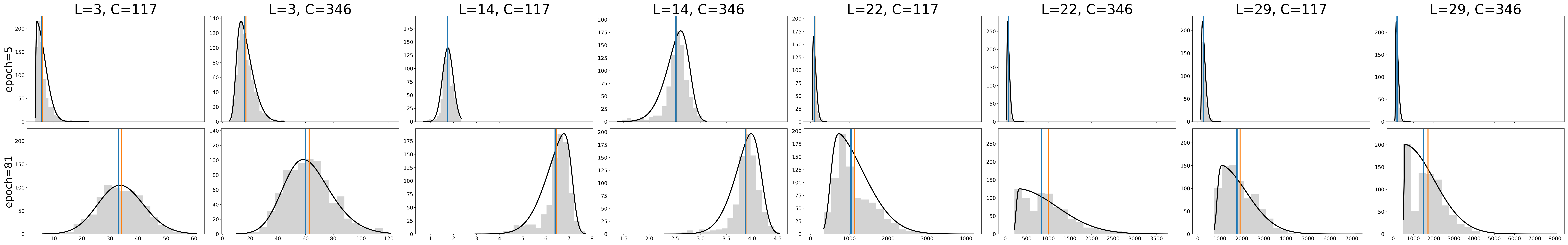}
 \includegraphics[width=\linewidth]{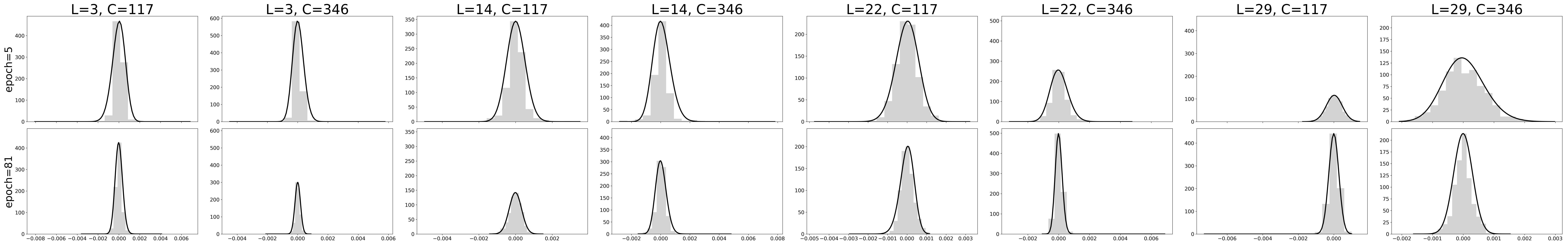}
  \caption{The activation (the 1st and 2nd rows) and gradient (the 3rd and 4th rows) statistics in channel C of normalization layer L. In activation statistics, we show the GM and AM of activation statistics in solid 'blue' and 'orange' lines respectively. }
  \label{fig:2-b}
\end{figure*}

\begin{algorithm}[!t]
\caption{ Fluctuation Smoothing }
\label{alg:un}

\raggedright
\textbf{Forward Propagation}
\begin{algorithmic}[1]
\Require $\mathbf{X}_t \in \mathbb{R}^{B \times C}$
\Ensure $\mathbf{Y}_t \in \mathbb{R}^{B \times C}$

\State $ \sigma^2_{t} = \frac{1}{B}{\textstyle \sum_{i=1}^{B}}\mathbf{x}^2_{t,i} $
\hfill \textcolor[rgb]{0.56,0.76,0.9}{$//$ mini-batch quadratic mean} 

\State  $ \hat{\sigma}_{t}^{2} = \sqrt[M]{{\textstyle \prod_{i=0}^{M-1}} \sigma_{t-i}^{2}} $
\hfill \textcolor[rgb]{0.56,0.76,0.9}{$//$ geometric mean} 

\State $\mathbf{Z}_{t} = \frac{ \mathbf{X}_{t} }{ \sqrt{\hat{\sigma}_{t}^{2}+\epsilon} }$
\hfill \textcolor[rgb]{0.56,0.76,0.9}{$//$ normalizing}

\State $ \mathbf{Y}_t = \gamma \cdot \mathbf{Z}_t + \beta $
\hfill \textcolor[rgb]{0.56,0.76,0.9}{$//$ re-scaling and shifting}

\State $ \sigma^{2} = \alpha \sigma^{2} + (1-\alpha ) \hat{\sigma}^{2}_{t} $
\hfill \textcolor[rgb]{0.56,0.76,0.9}{$//$ updating for inference}

\end{algorithmic}

\textbf{Backward Propagation}
\begin{algorithmic}[1]
\Require $\frac{\partial L}{\partial \mathbf{Y}_t } \in \mathbb{R}^{B \times C}$
\Ensure $\frac{\partial L}{\partial \mathbf{X}_t } \in \mathbb{R}^{B \times C}$

\State $ \frac{\partial L}{\partial \mathbf{Z}_{t}} = \gamma \cdot \frac{\partial L}{\partial \mathbf{Y}_{t}} $

\State $ g_{\hat{\sigma}^{2}_t} = \frac{1}{B} {\textstyle \sum_{i=1}^{B}\frac{\partial L}{\partial \mathbf{z}_{i}}\mathbf{z}_{i} } $
\hfill \textcolor[rgb]{0.56,0.76,0.9}{$//$ gradients from $ \hat{\sigma}^2_t $}

\State $ \psi_{\hat{\sigma}^2_t} = \alpha \psi_{\hat{\sigma}^2_{t-1}} + (1 - \alpha) \frac{1}{M} {\textstyle \sum_{i=0}^{M-1} g_{\hat{\sigma}^{2}_{t-i}}} $
\hfill \textcolor[rgb]{0.56,0.76,0.9}{$//$ estimating gradients}

\State $ \frac{\partial L}{\partial \mathbf{X}_{t}} = \frac{1}{\sqrt{\hat{\sigma}_{t}^{2} +\epsilon}} (\frac{\partial L}{\partial \mathbf{Z}_{t}} - \mathbf{Z}_{t} \cdot \psi_{\hat{\sigma}^{2}_t}) $

\end{algorithmic}

\end{algorithm}

\textbf{Analysis with Normality Test.} We dive deeper to analyze the abnormal behaviors of activation statistics in Transformers. To investigate the magnitude of fluctuations in activation and gradient statistics, we quantitatively analyze the abnormal behaviors with quadratic mean as the second-moment statistics. For statistics $\sigma_t^2 \in \mathbb{R}^{C}$, we conduct the normality test\cite{d1973tests} on a sequence of $\sigma_t^2, \cdots, \sigma_{t-M+1}^2$ to establish whether or not the sequence comes from a normally distributed population, then calculate the percentage of channels held for normality. We define the Percentage of Normality over All Channels (PNAC) which measure the degree of fluctuations in statistics:
\begin{equation}
PNAC=\frac{|\{c_i|\text{Normality Test}(c_i), p > 0.05 \}|}{C} \times 100\%, i=1,\cdots,C.
\label{eq:pnac}
\end{equation}
The lower PNAC, the larger fluctuations in statistics. In this way, we compute PNAC for each layer averaged over iterations and plot it in Figure~\ref{fig:2-a}. At the very beginning of training, both activation and gradient statistics mildly fluctuate over iterations. At the end of the training, there is a significant drop in the PNAC of activation statistics which means large fluctuations exist in activation statistics. On the contrary, we find that there are milder fluctuations in gradient statistics.

Moreover, we next visualize the activation and gradient statistics in different layers and channels, as shown in Figure~\ref{fig:2-b}. The range of activation statistics gets larger along the depth and training process. We find that the skewed distribution of activation statistics contains extreme outliers that could impact the arithmetic mean. We thus turn to adopt geometric mean (GM) with less sensitivity to outliers instead of arithmetic mean (AM) to gain a better representation of activation statistics in a skewed distribution. The averaging functions are defined as:

\begin{equation}
\hat{\mu}_t=0, \quad
\hat{\sigma}^2_t=\sqrt[M]{{\textstyle \prod_{i=0}^{M-1}} \sigma_{t-i}^{2}},
\label{eq:UN-theta-1}
\end{equation}

\begin{equation}
\psi_{\hat{\mu}_{t}}=0, \quad
\psi_{\hat{\sigma}^{2}_{t}}=\alpha \psi_{\hat{\sigma}^2_{t-1}} + (1- \alpha) \frac{1}{M} {\textstyle \sum_{i=0}^{M-1} g_{\hat{\sigma}^{2}_{t-i}}}.
\label{eq:UN-theta-2}
\end{equation}

\noindent By applying quadratic mean as the second-moment statistics, we visualize the GM and AM of activation statistics in Figure~\ref{fig:2-b}:in an approximately normal distribution (i.e., mild fluctuations), GM is close to AM; in skewed distribution (i.e., large fluctuations), the extreme outliers greatly influenced AM, while GM is still close to the majority. Owing to the gradient statistics $g_{\hat{\sigma}^2_{t}}$ are first-moment statistics and do not obey non-negativity constraints, it is unable to use GM directly. Therefore, we utilize AM in gradient statistics that further with a momentum for gradient estimation in backward propagation. Specially, we leverage the quadratic mean as the second-moment statistics in our method to reduce the number of statistics that ensure the stability in applying moving average strategies, as proved in~\cite{yan2020towards}. Our strategy could be formulated as shown in Algorithm~\ref{alg:un}. Omitting the impact of updated weights over different iterations, UN is set with moderate window sizes. 

\subsection{Outlier Filtration}

Although the fluctuation smoothing is leveraged to calibrate the activation statistics, extreme outliers are observed and somehow lead to instability in training (as shown in Figure~\ref{fig:1}). With the moving average strategies (in Equation~\ref{eq:Theta-mu} and ~\ref{eq:Theta-sigma}) applied to activation statistics, it is impossible to calculate the accurate gradients from previous iterations~\cite{huang2020normalization}. Once extreme outliers deteriorate the gradient estimation error, the risk of instability increases. Based on the assumption, we attempt to take a step further by introducing an adaptive outlier filtration strategy. More specifically, the main goal of outlier filtration is to decide when to apply the moving average strategies. To identify outliers, we set an adaptive threshold for outlier filtration with the \emph{$AM-GM$ inequality}~\cite{aldaz2012sharp}. Let $ \Omega_t = (\sigma^2_{t}, \sigma^2_{t - 1},\cdots, \sigma^2_{t-M+1})$ denote the $M$ recent activation statistics recorded in forward propagation at iteration $t$, where $M > 1$, then we have

\begin{equation}
E(\Omega_t) - \Pi(\Omega_t) \le M \cdot V(\Omega_t^{\frac{1}{2}}),
  \label{eq:10}
\end{equation}

\noindent where $\Omega_t^{\frac{1}{2}}=(\sigma_{t}, \cdots, \sigma_{t-M+1})$ and $V(\cdot)$, $E(\cdot)$, $\Pi(\cdot)$ are operators that calculate the variance, arithmetic mean, and geometric mean for input respectively. The extreme outliers will enlarge variances. We thus use the upper bound of the last iteration to detect outliers for the current iteration. Hence, $M \cdot V(\Omega_{t-1}^{\frac{1}{2}})$ can be used as an adaptive threshold for outlier filtration. That is to say, once the mini-batch is deemed to contain extremely large outliers and all the moving average strategies will be dropped in a specific normalization layer,
\begin{equation}
\begin{cases}
 \hat{\sigma}^2_t=\sigma^2_t, \quad \psi_{\hat{\sigma}^2_t}=g_{\hat{\sigma}^2_t} & \text{ if } E(\Omega_{t}) - \Pi(\Omega_t) > M \cdot V(\Omega_{t - 1}^{\frac{1}{2}})  \\
 \text{Equation (\ref{eq:UN-theta-1}) and (\ref{eq:UN-theta-2}) } & \text{ otherwise. }
\end{cases}
\label{eq:outlier}
\end{equation}
The current statistics $\sigma_{t}^{2}$ and $g_{\hat{\sigma}^{2}_{t}}$ will be used for forward-propagating and backward-propagating once outliers are found. In Equation~\ref{eq:outlier}, the threshold for outlier filtration is independent of the specific input $X_{t}$, making this strategy more adaptive to the ever-changing activation statistics during training. Besides, $\sigma^{2}$ and $\psi_{\hat{\sigma}^{2}_t}$ are used to update the recorded statistics at iteration $t$ for passivating the outliers in moving average. The rest of the operations are just the same as Algorithm~\ref{alg:un}.

\begin{lemma} 
Let $ A_t = ( a_t, a_{t - 1}, ..., a_{t-M+1} ) $ and $ A_{t - 1} = ( a_{t-1}, a_{t-2},\\ ..., a_{t-M} ) $ are two vectors satisfying $ a_i > 0 $ and $ a_i < a_t, \forall a_i \in A_{t - 1} $. $ E(\cdot) $, $ \Pi(\cdot) $, and $ V(\cdot) $ denote the calculation of arithmetic mean, geometric mean, and variance for an arbitrary vector. If $ M \cdot V(A_{t - 1}^\frac{1}{2}) < E(A_{t}) - \Pi(A_{t}) $ holds, then $ \lambda = \frac{\Pi(A_{t})}{a_t} < 1 $. 
\label{lemma1}
\end{lemma}

\begin{proof}

From the AM-GM inequality and the lemma condition, we have the following inequality
\begin{align}
    \Pi(A_{t}) - \Pi(A_{t - 1}) <  E(A_{t}) - E(A_{t - 1}). 
\end{align}

\noindent Since $ E(A_{t}) - E( A_{t - 1}) = \frac{a_t - a_{t - M}}{M} $, then
\begin{equation}
    M \cdot (\Pi(A_{t}) - \Pi(A_{t - 1})) < a_t - a_{t - M}.
    \label{ineq:12}
\end{equation}

\noindent With $a_t > 0$, the above inequality can be transformed to
\begin{equation}
    M \cdot \lambda (1 - \sqrt[M]{\frac{a_{t-M}}{a_t}}) < 1 - \frac{a_{t-M}}{a_t}. 
\end{equation}

\noindent Therefore, 
\begin{equation}
    \lambda < \frac{1}{M} \cdot \frac{1 - \frac{a_{t-M}}{a_t}}{1 - \sqrt[M]{\frac{a_{t-M}}{a_t}}} < \frac{1}{M} \cdot M = 1. 
\end{equation}

\noindent The last inequality holds because function $ f(x) = \frac{1 - x}{ 1 - \sqrt[M]{x}} $ is monotonically increasing in [0, 1).
\end{proof}

\begin{corollary}
For a Network using UN without outlier filtration, let $ g_{\sigma^2_t} $ denote the ground truth gradient computed based on the chain rule and $ \tilde{g}_{\sigma^2_t} $ denote estimated gradient given by a averaging function. If $ x_t $ contains an outlier, then

\label{corollary}
\end{corollary}

\begin{equation}
 \frac{g_{\sigma^2_t}}{\tilde{g}_{\sigma^2_t}} < \frac{1}{M}.
  \label{eq:x}
\end{equation}

\begin{proof}

UN uses geometric mean to estimate the training statistics, which modifies the current mini-batch statistics but detaches from the backward pass. Hence,

\begin{align}
    g_{\sigma^2_t} = \frac{ \partial{L} }{ \partial{\sigma^2_t} } = \frac{ \partial{L} }{ \partial{ \hat{\sigma}^2_t} } \cdot \frac{ \partial{\hat{\sigma}^2_t} }{ \partial{\sigma^2_t} } = \tilde{g}_{\sigma^2_t} \frac{ \partial{\hat{\sigma}^2_t} }{ \partial{ \sigma^2_t } }. 
\label{seq:9}
\end{align}

\noindent Combined with Lemma~\ref{lemma1}, we have

\begin{equation}
    \frac{g_{\sigma^2_t}}{\tilde{g}_{\sigma^2_t}} = \frac{ \partial{\hat{\sigma}^2_t} }{ \partial{\sigma^2_t} } = \frac{1}{M} \cdot \frac{\hat{\sigma}^2_t}{\sigma^2_t} 
    = \frac{1}{M} \cdot \frac{ \Pi(\Omega_t) }{ \sigma^2_t } < \frac{1}{M}.
\label{seq:10}
\end{equation}

\end{proof}

\noindent With the adaptive outlier filtration strategy proposed in this section, we can avoid a catastrophic gradient estimation error. Based on Corollary~\ref{corollary}, we prove that the gradient estimation error will be shrunk by a factor $ 1/M $ when an outlier is found. 

\section{Experiments}

\subsection{Implementation Details}
To put all experiments on an equal footing, we simply replace all LN in corresponding architectures with its drop-in counterparts, without varying the position of the normalization layer or adding extra operators. All models are trained and tested with the same configurations. To simplify the settings, the momentum of UN is set as $\alpha=0.9$ (the same as BN) to avoid repeatedly tuning hyper-parameter over different tasks. Akin to ~\cite{yan2020towards}, we set a warming-up step for UN, 4K by default. To show robust results, we report the average performance from 5-run results for small datasets, IWSLT14, CIFAR10, and CIFAR100. See Appendix~\ref{setting-details} for more experimental details.

\subsection{Results}

\begin{table}
  \centering
  \caption{The performance (BELU~\cite{papineni2002bleu}, higher is better) of Transformers on neural machine translation. ‘Offline’ indicates a method can be fused in inference. 'NoNorm' means models without normalization. 'FAIL' indicates collapse during training. PN$^*$ is PN without a layer-scale layer. }
  \label{tab:nmt}
  \small
  \begin{threeparttable}
  \begin{tabular}{l|c|cc|cc}
    \toprule
    \multirow{2}{*}{Method} & \multirow{2}{*}{Offline} &\multicolumn{2}{c}{IWSLT14} & \multicolumn{2}{|c}{WMT14}  \\
    \cline{3-6}
    & & BLEU & $\bigtriangleup$ & BELU & $\bigtriangleup$ \\
    \midrule
    LN~\cite{LayerNorm}    & \XSolidBrush & 35.3 &   &   40.0 &      \\
    RMSNorm~\cite{zhang2019root} &\XSolidBrush   & 35.3 & 0.0 &   39.8 & -0.2 \\
    PN~\cite{shen2020powernorm}  &\XSolidBrush   & 35.3 & 0.0 &   39.8 & -0.2 \\
    \midrule
    
    NoNorm & / &  FAIL & / &   32.8 & -7.2 \\
    \cline{1-6}
    BN~\cite{ioffe2015batch}  & \Checkmark &  31.1 & -4.2 &  35.1 & -4.9 \\
    MABN~\cite{yan2020towards}      &\Checkmark  &  \textbf{35.4}   & +0.1 &  36.5 & -3.5 \\
    PN$^*$~\cite{shen2020powernorm}    &\Checkmark &  35.0 & -0.3 &  39.7 & -0.3 \\
    UN                                 &\Checkmark &  \textbf{35.4} & +0.1 &  \textbf{39.9} & -0.1 \\
    \bottomrule
  \end{tabular}
  \end{threeparttable}
\end{table}

\begin{table}
  \centering
  \caption{The performance (Top-1 accuracy \%) of image classification on ImageNet-1K and CIFAR10/100. }
  \label{tab:cls}
  \small
  \resizebox{\linewidth}{!}{
  \begin{threeparttable}
    \begin{tabular}{l|c|cc|cc|cc|cc}
    \toprule
    \multirow{3}{*}{Method}  & \multirow{3}{*}{Offline} &  \multicolumn{2}{c}{Swin-T} & \multicolumn{6}{|c}{T2T-ViT-14$^\dagger$}  \\
    \cline{3-10}
    & & \multicolumn{2}{c}{ImageNet} & \multicolumn{2}{|c}{ImageNet} & \multicolumn{2}{|c}{CIFAR10} & \multicolumn{2}{|c}{CIFAR100} \\
    \cline{3-10}
    & & Top1 & $\bigtriangleup$ & Top1 & $\bigtriangleup$ & Top1 & $\bigtriangleup$ & Top1 & $\bigtriangleup$ \\
    \midrule
    LN~\cite{LayerNorm}   & \XSolidBrush  & 81.3 &   & 81.5 & & 98.3 & & 88.4 & \\
    \midrule
    BN~\cite{ioffe2015batch} & \Checkmark &  80.8 & -0.5 & 79.8 & -1.7 & 96.6 & -1.7 & 88.2& -0.2 \\
    MABN~\cite{yan2020towards} & \Checkmark& 80.9 & -0.4 & FAIL & /&/&/ &/ & /\\
    PN$^*$~\cite{shen2020powernorm}& \Checkmark& 80.9 & -0.4 & FAIL &/ &/ &/ & /&/ \\
    UN & \Checkmark& \textbf{81.0} & -0.3 & \textbf{80.9} & -0.6 & 98.3 & 0.0 &88.9 & +0.5 \\
    \bottomrule
  \end{tabular}
  \begin{tablenotes} 
  \item $\dagger$: On CIFAR10/100, models are initialized with pre-trained weights from ImageNet-1K. Note that T2T-ViT-14 cooperates with MABN and PN$^*$ crash during training on ImageNet, we thus do not report the corresponding results on CIFAR10/100.
  \end{tablenotes}
  \end{threeparttable}
  }
  
\end{table}

\begin{table}
  \centering
  \caption{Object detection on COCO val2017 with Faster R-CNN using Swin-T as the backbone. All models are trained with 36 epochs.}
  \label{tab:det-faster-rcnn}
  \small
  \resizebox{\linewidth}{!}{
   \begin{tabular}{l|c|cc|ccccc}
      \toprule
       Method & Offline &  $\text{AP}^{\text{Box}}$ & $\bigtriangleup$ & $\text{AP}^{\text{Box}}_{50}$ & $\text{AP}^{\text{Box}}_{75}$ & $\text{AP}^{\text{Box}}_{s}$ & $\text{AP}^{\text{Box}}_{m}$ & $\text{AP}^{\text{Box}}_{l}$  \\

	\midrule
        LN~\cite{LayerNorm}&\XSolidBrush & 45.5 & & 67.5 & 50.2 & 31.5 & 48.8 & 58.4  \\
        \midrule
		BN~\cite{ioffe2015batch}&  \Checkmark& 44.6 & -0.9 & 66.8 & 48.9 & \textbf{30.4} & 48.0 & 57.8 \\
		MABN~\cite{yan2020towards}& \Checkmark& 44.8 & -0.7 & 66.8 & 49.0 & 29.0 & 47.9 & 57.3 \\
	    PN$^*$~\cite{shen2020powernorm}& \Checkmark& 44.3 & -1.2 & 66.7 & 48.4 & 29.6 & 47.4 & 57.2 \\
		UN&	\Checkmark	& \textbf{45.2} & -0.3 & \textbf{67.2} & \textbf{49.7} & 30.1 & \textbf{48.4} & \textbf{58.2} \\
    \bottomrule
  \end{tabular}
  }
\end{table}

\subsubsection{Neural Machine Translation} The comparison between online and offline methods is listed in Table~\ref{tab:nmt}. Although RMSNorm~\cite{zhang2019root} and PN achieve comparable results with LN, online methods are not able to access efficient deployment on hardware. We report the performance of the Transformers trained without normalization layers (marked as 'NoNorm'). The instability in training and declined performance highlight the necessity of applying normalization. Besides, previous offline methods, such as BN, MABN, and PN$^*$ suffer from degradation of performance. Instead, our method outperforms other offline methods and achieves more balanced results that are on par with online methods on both IWSLT14 and WMT14.

\subsubsection{Image Classification} Table~\ref{tab:cls} reports the results of T2T-ViT-14 and Swin-T on ImageNet. In T2T-ViT-14, MABN and PN$^*$ that leverage vanilla moving average strategies experience divergence in training. BN appears instability since the early stage of training (shown in Figure~\ref{fig:1}), leading to degradation in accuracy. UN enjoys stability during training and obtains an improvement of $1.1\%$ over BN. In Swin-T, the top-1 accuracy of BN drops by $0.5\%$. UN surpasses other offline methods and restores the accuracy to $81.0\%$. After that, we evaluate UN on downstream classification tasks (CIFAR10/100). The training loss fluctuates dramatically and irregularly in T2T-ViT-14 finetuned with BN, leading to a significant drop in accuracy. UN converges stably without tuning any settings and outperforms LN on top-1 accuracy.

\begin{table*}
  \centering
  \caption{Object detection and semantic segmentation on COCO val2017 with Mask R-CNN using Swin-T as the backbone. All models are trained with 36 epochs.}
  \label{tab:det-mask-rcnn}
  \small
  \begin{tabular}{l|c|cc|ccccc|cc|ccccc}
    \toprule
    
      Method & Offline &  $\text{AP}^{\text{Box}}$ & $\bigtriangleup$ & $\text{AP}^{\text{Box}}_{50}$ & $\text{AP}^{\text{Box}}_{75}$ & $\text{AP}^{\text{Box}}_{s}$ & $\text{AP}^{\text{Box}}_{m}$ & $\text{AP}^{\text{Box}}_{l}$ &  $\text{AP}^{\text{Mask}}$ & $\bigtriangleup$ & $\text{AP}^{\text{Mask}}_{50}$ & $\text{AP}^{\text{Mask}}_{75}$ & $\text{AP}^{\text{Mask}}_{s}$ & $\text{AP}^{\text{Mask}}_{m}$ & $\text{AP}^{\text{Mask}}_{l}$ \\
	\midrule
	    LN~\cite{LayerNorm}	&\XSolidBrush & 46.0 & & 68.1 & 50.3 & 31.2 & 49.2 & 60.1 &
	    41.6 & & 65.1 & 44.9 & 25.9 & 45.1 & 56.9 \\
	    \midrule
		BN~\cite{ioffe2015batch} & \Checkmark& 44.9 & -1.1 & 67.2 & 49.0 & 29.6 & 48.4 & 58.3 
		& 40.8 & -0.8 &  64.0 & 43.8 & 24.9 & 44.4 & 55.3 \\ 
		MABN~\cite{yan2020towards} &\Checkmark &  45.1 &-0.9 & 67.2 & 49.6 & \textbf{30.0} & 48.3 & 57.7 
		& 41.0 &-0.6 & 64.2 & 44.1 & 24.9 & 44.7 & 55.0 \\
		PN$^*$~\cite{shen2020powernorm} &\Checkmark & 44.6 &-1.4 & 66.8 & 48.9 & 29.1 & 47.6 & 57.6 
		& 40.7 &-0.9 & 63.7 & 43.6 & 24.1 & 43.8 & 54.9 \\
		UN	&\Checkmark & \textbf{45.6} &-0.4 & \textbf{67.6} & \textbf{50.4} & 29.6 & \textbf{49.2} & \textbf{58.9} 
		& \textbf{41.4} & -0.2 &  \textbf{64.8} & \textbf{44.5} & \textbf{25.2} & \textbf{45.1} & \textbf{55.7} \\
    \bottomrule
  \end{tabular}
\end{table*}

\subsubsection{Object Detection and Instance Segmentation} In Table~\ref{tab:det-faster-rcnn}, our method restores the performance from other offline methods, with only a slight decrease of $0.3\%$ mAP compared to LN. In Table~\ref{tab:det-mask-rcnn}, the result also reveals that our method achieves comparable performance with LN. Here, we show UN surpasses other offline methods that draw strength from the fluctuation smoothing and outlier filtration. With these results, the conclusion easily comes to light that UN could be generalized to other vision tasks more than only image classification.

\subsection{Analysis}

\begin{table}
  \centering
  \caption{Ablation study on COCO val2017. Mask R-CNN with Swin-T is trained for 12 epochs.}
  \small
  \label{tab:ablation-study-2}
  \begin{threeparttable}
  \begin{tabular}{c|c|cc|cccc}
    \toprule
      \multirow{2}{*}{Exp} &  FP &  \multicolumn{2}{c|}{BP} &  \multicolumn{4}{|c}{COCO$^\ddagger$}\\
      \cline{2-8} 
      & GM    & AM &$\alpha$ & $\text{AP}^{\text{Box}}$ & $\Delta$ & $\text{AP}^{\text{Mask}}$ & $\Delta$\\
      \midrule
      1 & \Checkmark&\Checkmark&\Checkmark& 42.8  &   & 39.2 & \\
      \midrule
      2 & 		&\Checkmark&\Checkmark&  42.2 & -0.6 & 38.8 & -0.4 \\
      3 &\Checkmark&		&\Checkmark& 42.4 & -0.4 & 38.8 & -0.4 \\
      4 &\Checkmark&\Checkmark&		& 42.4 & -0.4 & 39.0 & -0.2 \\
      5 &\Checkmark&		&		& 42.5 & -0.3 & 39.0 & -0.2 \\
    \bottomrule
  \end{tabular}
  \end{threeparttable}
  
\end{table}

\begin{table}
  \centering
  \small
  \caption{The effect of using different window sizes ($M$) in UN is evaluated on COCO val2017.}
  \label{tab:win-size-cv}
  \begin{threeparttable}
  \begin{tabular}{l|ccccc}
    \toprule
    $M$ & 2 & 4 & 6 & 8  & 10  \\
    \midrule
    $\text{AP}^{\text{Box}}$ & 42.8 & 42.8  & 42.8 & 42.8 & 42.5  \\
    $\text{AP}^{\text{Mask}}$ & 39.2 & 39.3  & 39.1 & 39.2 & 39.2 \\
    \bottomrule
  \end{tabular}
  \end{threeparttable}
\end{table}

\begin{table}
  \centering
  \caption{The effect of using the outlier filtration is evaluated on ImageNet and COCO val2017. }
  \small    
  \label{tab:outlier-cv}
  \begin{tabular}{c|c|c|c|cc}
    \toprule
    \multirow{3}{*}{\thead{ Outlier \\ Filtration}} & T2T-ViT-14 & 
    \multicolumn{4}{c}{Swin-T}  \\
    \cline{2-6}
    & ImageNet & ImageNet & Faster RCNN & \multicolumn{2}{c}{Mask RCNN}  \\
    & Top1 & Top1 & $\text{AP}^{\text{Box}}$ & $\text{AP}^{\text{Box}}$ & $\text{AP}^{\text{Mask}}$ \\
    \midrule
    w/      & 80.9 & 81.0 & 42.2  & 42.8  & 39.2 \\
    w/o     & FAIL & 80.7 & FAIL  & FAIL  & FAIL \\
    \bottomrule
  \end{tabular}
\end{table}

\subsubsection{Ablation Study} In Table~\ref{tab:ablation-study-2}, we ablate UN to verify the contribution of the basic components. Note that in the fluctuation smoothing of UN, GM, AM, and momentum $\alpha$ are all used as moving average strategies. We remove the basic components one by one, as shown in Exp2-5. In Exp2, the performance deteriorates without any moving strategies applied for activation statistics. In gradient estimation, we conduct a compound moving strategy for gradient statistics, including arithmetic mean and momentum $\alpha$. In exp3-4, performance degrades after removing any basic part of the compound moving strategy, which implies both of them are contributed to better estimation for gradients. To compare to Exp2, we remove all moving strategies from gradient statistics in exp5. It turns out that models suffer from performance loss without gradient estimation. In this result, we show the fluctuation smoothing proposed in this paper has empowered UN to gain solid performance. 
Additionally, we also report the ablation study on IWSLT14, whose details can be found in Appendix~\ref{appendix-nmt}. On IWSLT14, the models can still benefit from the fluctuation smoothing for further improvement.

\subsubsection{Effect of the Window Size} We compare different window sizes $M \in \{2,4,6,8,10\}$ on COCO val2017. Table~\ref{tab:win-size-cv} reports the effect of window sizes in UN. When UN is set with moderate window sizes, the models converge stably and gain competitive performance at the end. 

\begin{figure}
  \centering
  \includegraphics[width=\linewidth]{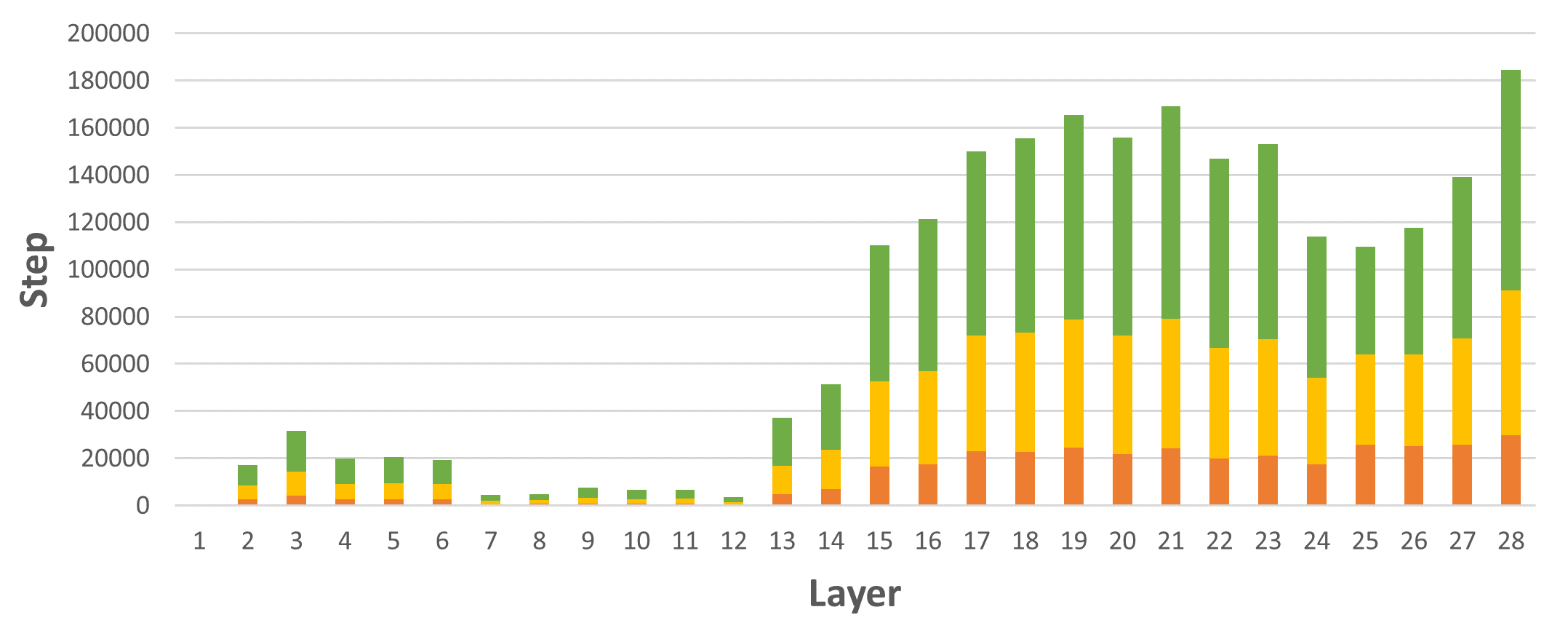}
  \caption{The accumulated steps of dropping moving averages in the outlier filtration during training. We plot the outcomes in red, orange, and green after training 12, 24, and 36 epochs, respectively.}
  \label{fig:3}
\end{figure}

\subsubsection{Outlier Filtration} As illustrated in Figure~\ref{fig:1}, it is easy to see that the outlier filtration stabilizes the training of UN in T2T-ViT while other offline methods tend to crash during training. Table~\ref{tab:outlier-cv} showcases the results with and without outlier filtration. In Transformers, the fluctuations in activation statistics increase along with the depth. Figure~\ref{fig:3} shows the accumulated steps of iterations that have found outliers. Outliers tend to emerge from a deeper layer. As the model is trained with more epochs, the percentage increases. The observation implies that Transformers trained with offline methods might get larger fluctuations when scaling up the depth. This result also reveals that the fluctuations also increase along with the training process.

\begin{figure}
  \centering
  \includegraphics[width=\linewidth]{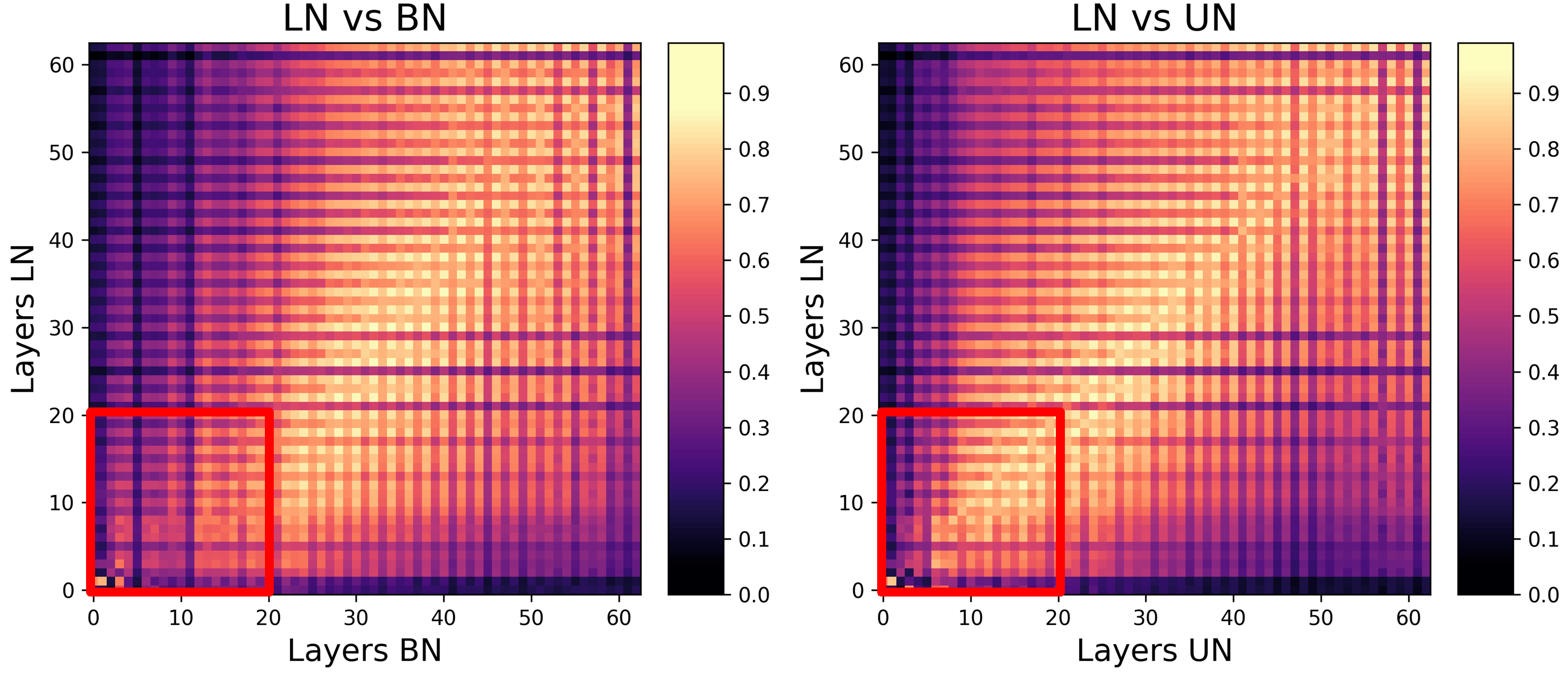}
  \caption{Comparing the similarity of feature maps across T2T-ViT-14 (ImageNet) with shallow layers highlighted in the red box. The similarity is measured with Centered Kernel Alignment (CKA)~\cite{nguyen2020wide} over layers including normalization layers, MHSA, and FFN in T2T-ViT-14. All LN layers in T2T-ViT-14 are simply replaced with BN and UN.}
  \label{fig:4}
\end{figure}

\subsubsection{Feature Similarity} With employing Centered Kernel Alignment (CKA), a widely-used representation similarity metric, we can study the internal representations between different models. As depicted in Figure~\ref{fig:4}, we compare the similarity of intermediate feature maps between models trained with different normalization methods. We mainly focus on the diagonal pixels in the heatmap that indicate the similarity across layers of the same depth. UN shows a more remarkable similarity with LN compared to BN, especially in shallow layers. This result somehow explains the superiority of our method.

\begin{table}
  \centering
  \caption{Inference efficiency comparison between LN and UN in Swin-T. MEM (MB) is the maximum allocated memory during inference. TPUT (Img./Sec.) shows the average throughput calculated over 1000 batches. We set a batch size of 512 for ImageNet and 2 for COCO (with Mask R-CNN).}
  \small
  \label{tab:efficiency}
   \begin{tabular}{c|l|cc|cc}
    \toprule
        Task & Method & MEM & Reduction & TPUT & Speedup \\
     \midrule
        \multirow{2}{*}{ImageNet} 
        & LN        & 9978 &  -         & 1179.5 & -      \\
		& UN        & 8213 &  17.7\%    & 1547.8 & 31.2\% \\
	\midrule
        \multirow{2}{*}{COCO} 
        & LN        & 955 &  -          & 17.8 & -      \\
		& UN        & 897 &  6.1\%      & 22.1 & 24.2\% \\
    \bottomrule
  \end{tabular}
\end{table}

\subsubsection{Efficiency} Transformers are broadly applied to vision tasks and attempt to achieve efficient deployment. LN comes with an additional overhead of computation and memory that results in inefficient inference. Besides, LN can not be supported on many edge devices,~\eg, NXP i.MX Series and TITDA4x. There is still room for Transformers to be further improved to achieve hardware-efficient deployment. In this paper, we focus on improving the normalization layer for Transformers in order to achieve a better trade-off between performance and inference speed. With fusing UN to other linear layers, the division and square root operations are also removed in inference. Experimentally, we test the efficiency on GeForce RTX 3090 with Swin-T as reported in Table~\ref{tab:efficiency}. For classification, we show that when our method is fused with other linear operations, it gains \textbf{about 18\%} memory reduction and \textbf{over 31\%} throughput improvement. For object detection, Mask R-CNN with Swin-T is integrated with other components like FPN and head, whereas LN is solely employed in the backbone. As a result, the increase in speed is limited.

\section{Conclusion}

In this paper, we look at how to deploy Transformers efficiently by replacing LN with an offline method. Previous offline methods suffer from inferior performance and instability due to the large fluctuations and extreme outliers in activation statistics. Based on our analysis, we propose UN that consists of the fluctuation smoothing and the outlier filtration strategies to tackle these challenges. Extensive experiments on NLP and CV tasks demonstrate that our method significantly outperforms previous offline methods. Furthermore, our method provides comparable performance to LN, with a speedup of over 31\% in inference.  We believe our method will be a general component in Transformers for efficient deployment.

\begin{acks}
This work was supported by the National Natural Science Foundation of China (No. U19B2043).
\end{acks}

\bibliographystyle{ACM-Reference-Format}
\bibliography{acmmm22-manuscript}


\clearpage
\appendix

\section{Visualization}

\subsection{Large Fluctuations in Activation Statistics}

\begin{figure}
  \centering
  \includegraphics[width=\linewidth]{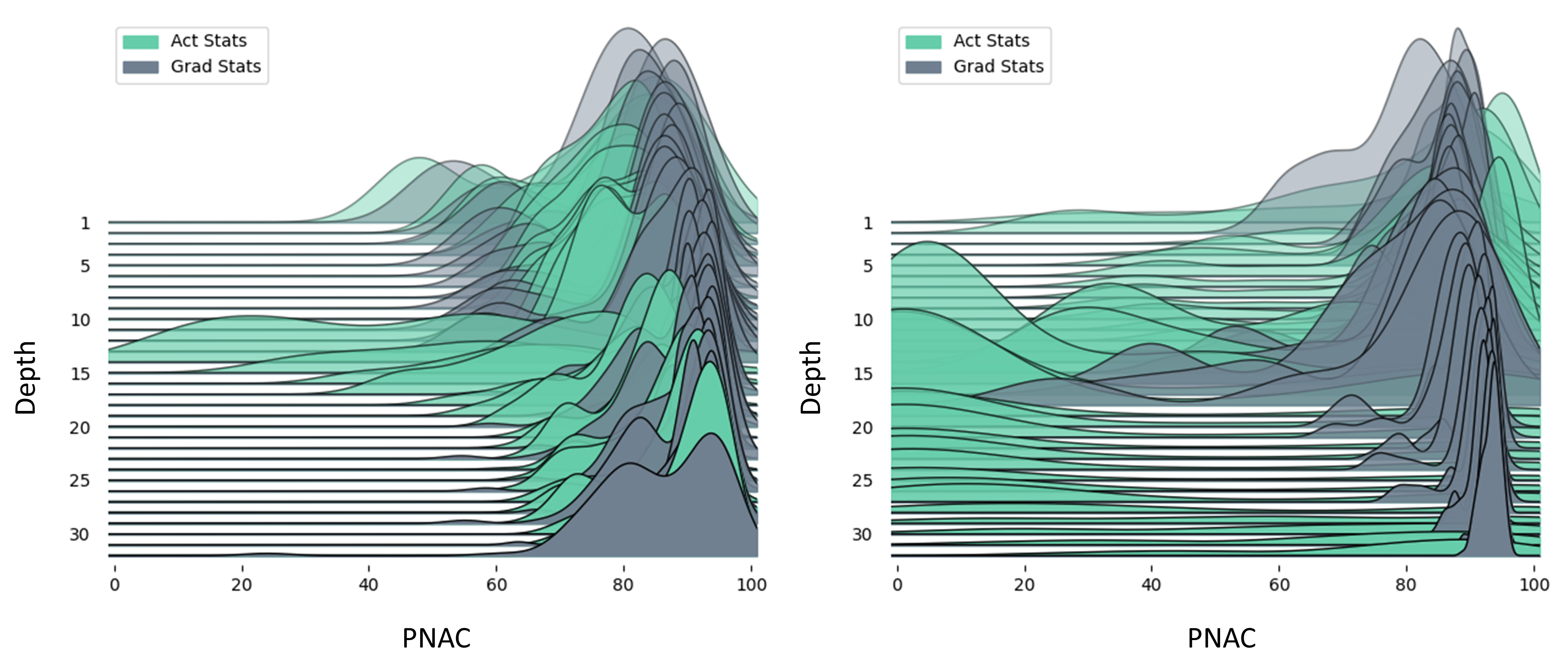}
   \caption{The probability density map of PNAC in different layers (depth) over iterations. Left side: visualization at the beginning of training. Right side: visualization at the end of training. We collect the samples from 150 time steps. The kernel density estimate (KDE) is used for visualization. Note that activation and gradient statistics are plotted in 'aquamarine' and 'gray', respectively.}
  \label{sup-fig:1}
\end{figure}

Figure~\ref{sup-fig:1} showcases the probability density map of PNAC during training that illustrates a general trend of statistics for all normalization layers in Transformer. At the beginning of training, most channels in activation and gradient statistics hold a high PNAC, which indicates mild fluctuations. After training a couple of epochs, there is a huge shift in activation statistics that a lot of channels gain a much lower PNAC, which means large fluctuations emerge in activation statistics. In this paper, a tailored fluctuation smoothing strategy is utilized to gain a better representation of these ever-changing statistics.

\section{Details for Experiments}
\label{setting-details}
\subsection{Setup for Neural Machine Translation} 
We evaluate our method with Transformer on two datasets: (1) IWSLT14 De-En (IWSLT14) contains 0.18M sentence pairs; (2) WMT14 En-Fr (WMT14)~\cite{wu2016google} contains 36M sentence pairs. The setup for prepossessing raw data is the same as~\cite{mehta2020delight}. For IWSLT14, we replicate the training and evaluation strategies in~\cite{liu2020admin}. For WMT14, we follow the training and evaluation setup in~\cite{mehta2020delight} and average the last $5$ checkpoints for the test. Here, all experiments are re-implemented on the code base of Fairseq~\cite{ott2019fairseq} with pre-normalization~\cite{xiong2020layer} setting.

\subsection{Setup for Image Classification}
In this section, we conduct image classification on ImageNet-1K~\cite{deng2009imagenet} with two state-of-the-art Vision Transformers, T2T-ViT-14~\cite{yuan2021tokens} and Swin-T~\cite{liu2021swin}. ImageNet-1K is a widely-used image classification dataset, which contains 1000 categories, 1.28M training samples, and 50K validation samples. Following the setup in ~\cite{yuan2021tokens,liu2021swin}, all models are trained from scratch for 300 epochs with a cropped input size of $224\times224$. We replace all LN layers in original architectures with BN/MABN/PN and our proposed method UN. After pretraining models on ImageNet, we transfer the models to downstream classification datasets, CIFAR10 and CIFAR100, that focus on general object classification. We follow the training recipe in~\cite{yuan2021tokens}. All models are fed with a resized $224\times224$ input and finetuned for 60 epochs.

\subsection{Setup for Object Detection and Instance Segmentation} 
We benchmark our method on COCO~\cite{lin2014microsoft}. Following the standard setup in ~\cite{liu2021swin}, object detection is conducted on Faster R-CNN~\cite{ren2015faster} with FPN~\cite{lin2017feature}. For instance segmentation, we evaluate our method with a common framework Mask R-CNN~\cite{he2017mask}. The setup for training and evaluation are following the original configurations on ~\cite{liu2021swin}, all models are trained with 36 epochs. The input size is $1333\times800$ and the total batch size is set as $16$. The backbone (Swin-T) is initialized with pretrained weights trained on ImageNet-1K. All experiments are re-implemented based on {\tt mmdetection}~\cite{chen2019mmdetection}. 

\section{Extra Results}

\subsection{Ablation Study on Neural Machine Translation}
\label{appendix-nmt}

\begin{table}
  \centering
  \caption{Ablation of different components of UN evaluated on IWSLT14 with Transformer.}
  \label{tab:ablation-study-1}
  \begin{tabular}{cl|c|cc|cc}
    \toprule
      \multirow{2}{*}{Exp} & \multirow{2}{*}{Method}  &  FP &  \multicolumn{2}{c|}{BP} &  \multicolumn{2}{c}{IWSLT14}\\
      \cline{3-7} 
      &         & GM    & AM &$\alpha$ & BLEU & $\Delta$    \\
      \midrule
      1 & UN	&\Checkmark&\Checkmark&\Checkmark& 35.4 & \\
      \midrule
      2 & 		&		&\Checkmark&\Checkmark&34.5	& -0.9  \\
      3 & 		&\Checkmark&		&\Checkmark&35.3	& -0.1\\
      4 & 		&\Checkmark&\Checkmark&		&18.3   & -17.1 \\
      5 & 		&\Checkmark&		&		&FAIL   & /     \\
    \bottomrule
  \end{tabular}
\end{table}

Table~\ref{tab:ablation-study-1}, we ablate UN on IWSLT14 to verify the contribution of the basic components. Since Geometric Mean (GM), Arithmetic Mean(AM), and momentum $\alpha$ are all used in the fluctuation smoothing, we try to remove the basic components one by one, as shown in Exp2-5. Without any moving strategies applied in activation statistics, there is a significant drop on BELU. In backward propagation, we conduct a compound moving strategy for gradient estimation, which consists of arithmetic mean and momentum $\alpha$. The results exhibit that each one of them is important for the final performance. Especially, the momentum $\alpha$ is greatly helpful for stabilizing the training on ISWLT14. In Exp5, once we remove all moving strategies in BP, the model will fail to converge. The result shows the advantage of the fluctuation smoothing in IWSLT14.

\subsection{Effect on $\alpha$}

\begin{table}
  \centering
  \small
  \caption{The effect of $\alpha$ in UN is evaluated on COCO val2017 (AP) and IWSLT14 (BELU).}
  \label{tab:alpha-cv-nlp}
  \begin{threeparttable}
  \begin{tabular}{l|cccc}
    \toprule
    $\alpha$ & 0.9 & 0.8 & 0.7 & 0.6\\
    \midrule
    $\text{AP}^{\text{Box}}$ & 42.8 & 42.6 & 42.7 & 42.8 \\
    $\text{AP}^{\text{Mask}}$ & 39.2 & 39.1 & 39.2 & 39.4  \\
    \midrule
    BELU & 35.4 & 35.4 & 35.3 & FAIL  \\
    \bottomrule
  \end{tabular}
  \end{threeparttable}
\end{table}

In UN, we leverage a momentum $\alpha$ for both approximating inference statistics in forward propagation and estimating gradients in backward propagation. We investigate the effect of $\alpha$ as shown in Table~\ref{tab:alpha-cv-nlp}. By tuning $\alpha$ within a large range of $\{0.6, 0.7, 0.8, 0.9\}$, we find the models still converge stably on COCO with close performance. Choosing $\alpha$ from $\{0.7, 0.8, 0.9\}$ is also fine with IWSLT14. Once we set it with a small ratio, such as $\alpha=0.6$, the results are task-specific in that the training collapsed on IWSLT14. To some degree, that’s similar to what we show in Table~\ref{tab:ablation-study-1} (remove the momentum in BP). As a result, we believe that $\alpha=0.9$ would be a good choice for various tasks while also allowing for a fair comparison with other methods.

\end{document}